\documentclass[runningheads]{llncs}
\usepackage{graphicx}
\usepackage{enumitem}
\usepackage{booktabs} 
\usepackage{xspace}
\usepackage[mathscr]{euscript}
\usepackage{algorithm, algorithmicx}
\usepackage[noend]{algpseudocode}
\usepackage{ amssymb }
\usepackage{amsmath}
\usepackage{color}
\usepackage{caption}
\usepackage{subcaption}
\usepackage{nicefrac}       
\usepackage{wrapfig}
\usepackage{tikz}
\usepackage{tikz-cd}
\usepackage{bbm}
\usepackage{mathabx,graphicx}

\usepackage{microtype}

\DeclareMathAlphabet{\mathpzc}{OT1}{pzc}{m}{it}

\newcommand{\eg}{{e.g.}\xspace}
\newcommand{\etc}{{etc.}\xspace}

\newcommand{\ignore}[1]{}

\newcommand{\setof}[1]{\ensuremath{\left\{#1\right\}}}

\newcommand{\mypara}[1]{\noindent{\bf #1.}}

\newcommand{\Reals}{\mathbb{R}}

\newcommand{\PosZReals}{\mathbb{R}^{\ge 0}} 
\newcommand{\Nat}{\mathbb{N}}

\newcommand{\PowerSet}[1]{\mathscr{P}\left (#1 \right)}

\newcommand{\domain}{\mathcal{D}}
\newcommand{\timedomain}{T}

\newcommand{\f}{\varphi}

\newcommand{\validitydomain}{\mathcal{V}}
\newcommand{\validitydomainboundary}{\partial\mathcal{V}}

\newcommand{\pOrder}{\trianglelefteq}
\newcommand{\eqdef}{\mathrel{\stackrel{\makebox[0pt]{\mbox{\normalfont\tiny def}}}{=}}}

\newcommand{\trc}{\mathbf{x}}
\newcommand{\trcY}{\mathbf{y}}

\DeclareMathOperator*{\argmax}{arg\,max}

\usepackage[colorinlistoftodos]{todonotes}

\setlist{itemsep=1pt, topsep=2pt}

\begin{document}
\title{Time-Series Learning using Monotonic Logical
  Properties}
%
%
\author{Marcell Vazquez-Chanlatte\inst{1} \and
Shromona Ghosh\inst{1} \and
Jyotirmoy V. Deshmukh\inst{2} \and
Alberto Sangiovanni-Vincentelli \inst{1} \and
Sanjit A. Seshia \inst{1}}

\authorrunning{Vazquez-Chanlatte et al.}
%
\institute{University of California, Berkeley, USA\\
\email{\{marcell.vc, shromona.ghosh, alberto, sseshia\}@eecs.berkeley.edu}\\
\and
University of Southern California, USA\\
\email{jdeshmuk@usc.edu}}
\maketitle              

\begin{abstract}
   Cyber-physical systems of today
  are generating large volumes of time-series data. As manual
  inspection of such data is not tractable, the need for learning
  methods to help discover logical structure in the data has
  increased.  We propose a logic-based framework that allows
  domain-specific knowledge to be embedded into formulas in a
  parametric logical specification over time-series data. The key idea
  is to then map a time series to a surface in the parameter space
  of the formula. Given this mapping, we identify the Hausdorff
  distance between boundaries as a natural distance metric
  between two time-series data under the lens of the parametric
  specification.  This enables embedding non-trivial domain-specific
  knowledge into the distance metric and then using off-the-shelf
  machine learning tools to label the data. After labeling the data,
  we demonstrate how to extract a logical specification for each
  label. Finally, we showcase our technique on real world traffic
  data to learn classifiers/monitors for slow-downs and traffic jams.
\end{abstract}


\section{Introduction}\label{sec:intro}
Recently, there has been a proliferation of sensors that monitor
diverse kinds of real-time data representing {\em time-series
  behaviors\/} or {\em signals\/} generated by systems and devices
that are monitored through such sensors. However, this deluge can
place a heavy burden on engineers and designers who are not interested
in the details of these signals, but instead seek to discover
higher-level insights in the data.

More concisely, one can frame the key challenge as: ``How does one
automatically identify logical structure or relations within the
data?'' To this end, modern machine learning (ML) techniques for
signal analysis have been invaluable in domains ranging from
healthcare analytics~\cite{kale2014examination} to smart
transportation~\cite{deng2016latent}; and from autonomous
driving~\cite{mccall2007driver} to social media~\cite{liu_sparse}.
However, despite the success of ML based techniques, we believe that
easily leveraging the domain-specific knowledge of non-ML experts
remains an open problem.

At present, a common way to encode domain-specific knowledge into an
ML task is to first transform the data into an {\em a priori\/} known
{\em feature space}, e.g., the statistical properties of a time
series. While powerful, translating the knowledge of domain-specific
experts into features remains a non-trivial endeavor.  More recently,
it has been shown that Parametric Signal Temporal Logic formula along
with a total ordering on the parameter space can be used to extract
feature vectors for learning temporal logical predicates
characterizing driving patterns, overshoot of diesel engine re-flow
rates, and grading for simulated robot controllers in a Massively Open
Online Course~\cite{logicalClustering}. Crucially, the technique of
learning through the lens of a logical formula means that learned
artifacts can be readily leveraged by existing formal methods
infrastructure for verification, synthesis, falsification, and
monitoring. Unfortunately, the usefulness of the results depend
intimately on the total ordering used. The following example
illustrates this point.

\begin{figure}[h]
  \centering \includegraphics[width=3in]{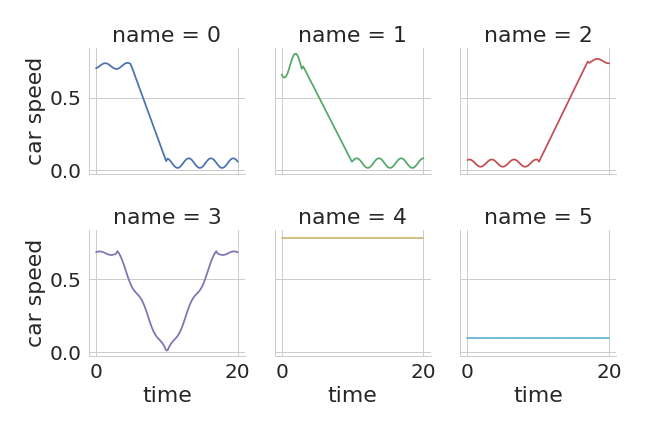}
  \caption{Example signals of car speeds on a
    freeway.}\label{fig:example_traces}
\end{figure}
\mypara{Example:} Most freeways have bottlenecks that lead to traffic
congestion, and if there is a stalled or a crashed vehicle(s) at this
site, then upstream traffic congestion can severely
worsen.\footnote{We note that such data can be obtained from fixed
  mounted cameras on a freeway, which is then converted into
  time-series data for individual vehicles, such as in~\cite{NGSIM}.}
For example, Fig~\ref{fig:example_traces} shows a series of potential
time-series signals to which we would like to assign pairwise
distances indicating the similarity (small values) or differences
(large values) between any two time series. To ease exposition, we
have limited our focus to the car's speed. In signals 0 and 1, both
cars transition from high speed freeway driving to stop and go
traffic.  Conversely, in signal 2, the car transitions from stop and
go traffic to high speed freeway driving. Signal 3 corresponds to a
car slowing to a stop and then accelerating, perhaps due to difficulty
merging lanes. Finally, signal 4 signifies a car encountering no
traffic and signal 5 corresponds to a car in heavy traffic, or a
possibly stalled vehicle.

Suppose a user wished to find a feature space equipped with a measure
to distinguish cars being stuck in traffic. Some properties might be:
\begin{enumerate}[leftmargin=1.5em,itemsep=0pt,parsep=0pt,partopsep=0pt]
\item Signals 0 and 1 should be {\em very\/} close together since both
  show a car entering stop and go traffic in nearly the same manner.
\item Signals 2, 3, and 4 should be close together since the car
  ultimately escapes stop and go traffic.
\item Signal 5 should be far from all other examples since it does not
  represent entering or leaving stop and go traffic.
\end{enumerate}
\begin{figure}[t]
  \centering
  \begin{subfigure}{0.49\textwidth}
    \centering
    \includegraphics[width=\textwidth]{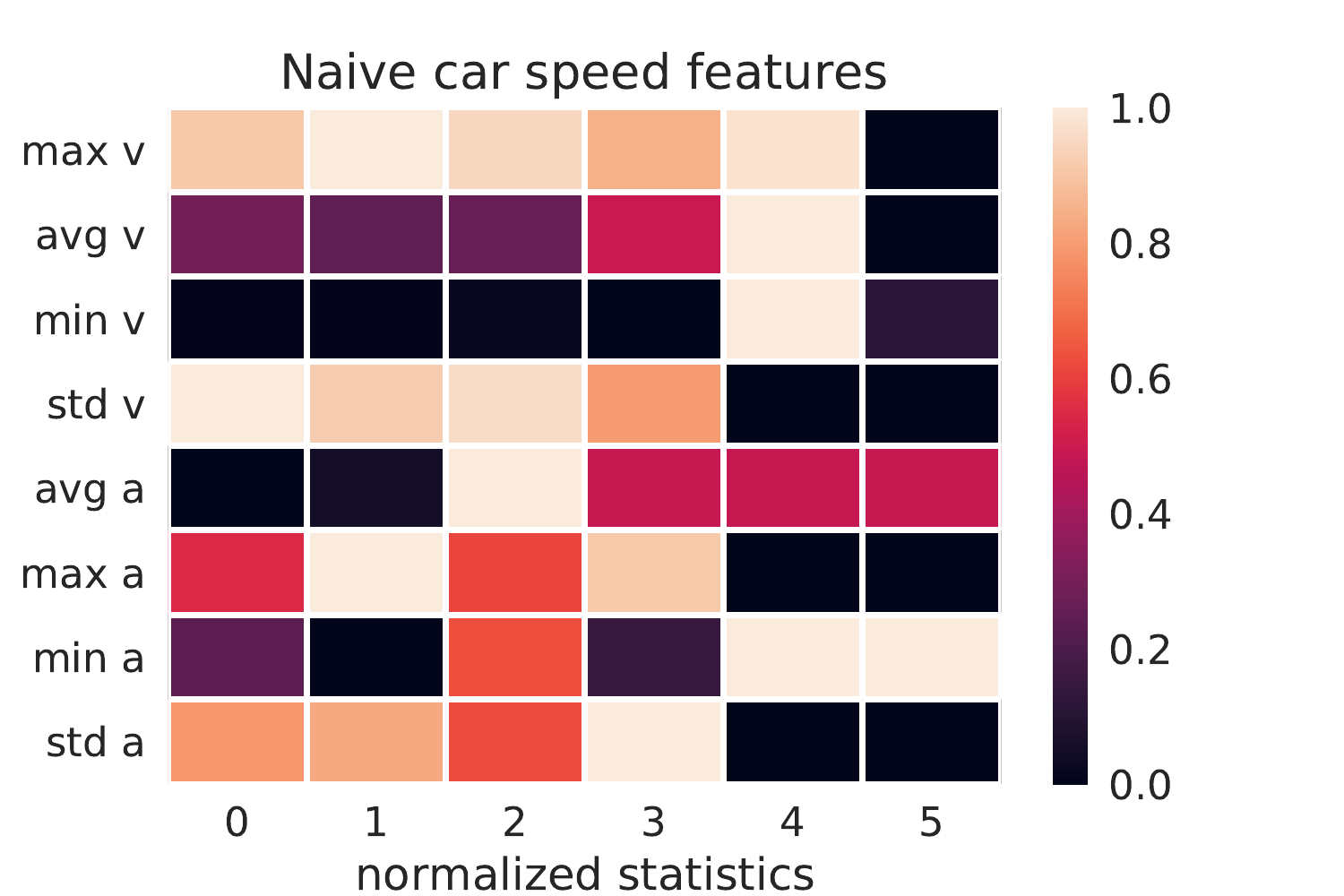}
    \caption{Statistical feature space}\label{fig:normalizedStats}
  \end{subfigure}
  \begin{subfigure}{0.49\textwidth}
    \centering
    \includegraphics[width=\linewidth]{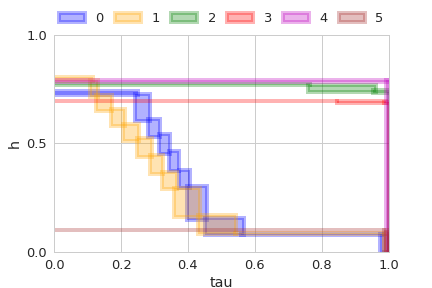}
    \caption{Trade-off boundaries in
      specification.}\label{fig:toy_boundaries}
  \end{subfigure}
\end{figure}

  \begin{wrapfigure}{r}{0.4\textwidth}
    \centering\fbox{\includegraphics[width=0.4\textwidth]{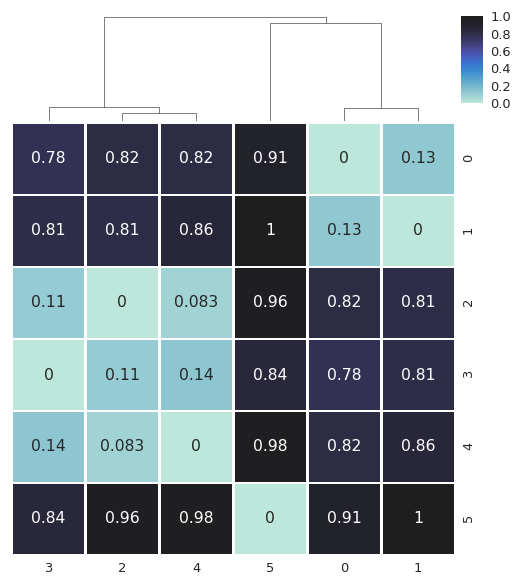}}
    \caption{Adjacency matrix and clustering of
      Fig~\ref{fig:example_traces}. Smaller numbers mean that the time
      series are more similar with respect to the logical distance
      metric.\label{fig:toy_cluster_map}}
  \end{wrapfigure}
  
  For a strawman comparison, we consider two ways the user might
  assign a distance measure to the above signal space. Further, we
  omit generic time series distance measures such as Dynamic Time
  Warping~\cite{keogh2000scaling} which do not offer the ability to
  embed domain specific knowledge into the metric.  At first, the user
  might treat the signals as a series of independent measurements and
  attempt to characterize the signals via standard statistical
  measures on the speed and acceleration (mean, standard deviation,
  \etc). Fig~\ref{fig:normalizedStats} illustrates how the example
  signals look in this feature space with each component normalized
  between 0 and 1. The user might then use the Euclidean distance of
  each feature to assign a distance between signals.  Unfortunately,
  in this measure, signal 4 is not close to signal 2 or 3, violating
  the second desired property. Further, signals 0 and 1 are not
  ``very'' close together violating the first property.  Next, the
  user attempts to capture traffic slow downs by the following
  (informal) parametric temporal specification: ``Between time $\tau$
  and 20, the car speed is always less than $h$.''.  As will be made
  precise in the preliminaries (for each individual time-series)
  Fig~\ref{fig:toy_boundaries}
  illustrates the boundaries between values of $\tau$ and $h$ that
  make the specification true and values which make the specification
  false. The techniques in~\cite{logicalClustering} then require the
  user to specify a particular total ordering on the parameter space.
  One then uses the maximal point on the boundary as the
  representative for the entire boundary. However, in practice,
  selecting a good ordering a-priori is non-obvious. For
  example,~\cite{logicalClustering} suggests a lexicographic ordering
  of the parameters. However, since most of the boundaries start and
  end at essentially the same point, applying any of the lexicographic
  orderings to the boundaries seen in Fig~\ref{fig:toy_boundaries}
  would result in almost all of the boundaries collapsing to the same
  points. Thus, such an ordering would make characterizing a slow down
  impossible.

  In the sequel, we propose using the Hausdorff distance between
  boundaries as a general ordering-free way to endow time series with
  a ``logic respecting distance
  metric''. Fig~\ref{fig:toy_cluster_map} illustrates the distances
  between each boundary. As is easily confirmed, all 3 properties
  desired of the clustering algorithm hold.

  \mypara{Contributions} The key insight in our work is that in many
  interesting examples, the distance between satisfaction boundaries
  in the parameter space of parametric logical formula can
  characterize the domain-specific knowledge implicit in the
  parametric formula.  Leveraging this insight we provide the
  following contributions:
  \begin{enumerate}[leftmargin=1.5em,itemsep=0pt,parsep=0pt,partopsep=0pt]
  \item We propose a new distance measure between time-series
    through the lens of a chosen monotonic specification.  Distance
    measure in hand, standard ML algorithms such as nearest neighbors
    (supervised) or agglomerative clustering (unsupervised) can be
    used to glean insights into the data.
  \item Given a labeling, we propose a method for computing
    representative points on each boundary. Viewed another way, we
    propose a form of dimensionality reduction based on the temporal
    logic formula.
  \item Finally, given the representative points and their labels, we
    can use the machinery developed in~\cite{logicalClustering} to
    extract a simple logical formula as a classifier for each label.
  \end{enumerate}


\section{Preliminaries}\label{sec:prelim}
The main object of analysis in this paper are time-series.\footnote{Nevertheless, the material presented in the sequel easily generalizes to other objects.}
\begin{definition}[Time Series, Signals, Traces] Let $\timedomain$ be a
  subset of $\PosZReals$ and $\domain$ be a nonempty
  set.  A time series (signal or trace), $\trc$ is
  a map:
  \begin{equation}
    \trc: \timedomain \to \domain
  \end{equation}
  Where $\timedomain$ and $\domain$ are called the time domain and
  value domain respectively.  The set of all time series is denoted
  by $\domain^\timedomain$.
\end{definition}
Between any two time series one can define a metric which measures
their similarity.
\begin{definition}[Metric]
  Given a set $X$, a metric is a map,
  \begin{equation}
    \label{eq:metric}
    d : X \times X \to \PosZReals
  \end{equation}
  such that $d(x, y) = d(y,x)$, $d(x, y) = 0 \iff x = y$,
  $d(x, z) \leq d(x, y) + d(y, z)$.
\end{definition}
\begin{example}[Infinity Norm Metric]
  The infinity norm induced distance $d_\infty(\vec{x}, \vec{y})
  \eqdef \max_i\left({|x_i - y_i|}\right)$ is a metric.
\end{example}
\begin{example}[Hausdorff Distance]
  Given a set $X$ with a distance metric $d$, the Hausdorff distance
  is a distance metric between closed subsets of $X$. Namely, given
  closed subsets $A, B \subseteq X$:
  \begin{equation}\label{eq:hausdorff}
    d_H(A, B) \eqdef \max\left(\sup_{x \in A}~\inf_{y \in B}(d(x, y)), \sup_{y \in B}~\inf_{x \in A}(d(y, x))\right )
  \end{equation}
  We use the following property of the Hausdorff distance
  throughout the paper: Given two sets $A$ and $B$, 
  there necessarily exists points $a\in A$ and $b\in B$ such
  that:
  \begin{equation}
    d_H(A, B) = d(a, b)
  \end{equation}
\end{example}

Within a context, the platonic ideal of a metric between traces
respects any domain-specific properties that make two elements
``similar''.\footnote{Colloquially, if it looks like a duck and
quacks like a duck, it should have a small distance to a duck.} A
logical trace property, also called a specification, assigns to each
timed trace a truth value.
\begin{definition}[Specification]
  A specification is a map, $\phi$, from time series to true or false.
  \begin{equation}
    \phi : \domain^\timedomain \to \{1, 0\}
  \end{equation}
  A time series, $\trc$, is said to satisfy a specification iff
  $\phi(\trc) = 1$.
\end{definition}
\begin{example}\label{ex:spec}
  Consider the following specification related to the specification
  from the running example:

  \begin{equation}    
    \phi_{ex}(\trc) \eqdef
    \mathbbm{1}\bigg [\forall t \in \timedomain ~.~ \big (t > 0.2
    \implies \trc(t) < 1 \big)\bigg](\trc)
  \end{equation}
  where $\mathbbm{1}[\cdot]$ denotes an indicator function.  Informally, this specification says that after
  $t=0.2$, the value of the time series, $x(t)$, is always
  less than 1.
\end{example}
Given a finite number of properties, one can then ``fingerprint'' a
time series as a Boolean feature vector. That is, given $n$
properties, $\phi_1 \ldots \phi_n$ and the corresponding indicator
functions, $\phi_1 \ldots \phi_n$, we map each time series to
an $n$-tuple as follows.
\begin{equation}
  \trc \mapsto (\phi_1(\trc), \ldots, \phi_n(\trc))
\end{equation}

Notice however that many properties are not naturally captured by a
{\em finite\/} sequence of binary features. For example, imagine a
single quantitative feature $f : \domain^\timedomain \to [0, 1]$
encoding the percentage of fuel left in a tank. This feature
implicitly encodes an uncountably infinite family of Boolean features
$\phi_k(\trc) = \mathbbm{1}[f(\trc) = k](\trc)$ indexed by the percentages $k \in
[0, 1]$. We refer to such families as parametric specifications. For
simplicity, we assume that the parameters are a subset of the unit
hyper-box.
\begin{definition}[Parametric Specifications] A parametric
  specification is a map:
  \begin{equation}\label{eq:pspec}
    \f: \domain^\timedomain \to \bigg ({[0, 1]}^n \to \{0, 1\} \bigg )
  \end{equation}
  where $n\in \Nat$ is the number of parameters and
  $\bigg ({[0, 1]}^n \to \{0, 1\} \bigg )$ denotes the set of
  functions from the hyper-square, ${[0, 1]}^n$ to $\{0, 1\}$.
\end{definition}
\begin{remark}
  The signature, $\varphi : {[0, 1]}^n \to (\domain^\timedomain \to \setof{0, 1})$ would have been an alternative
  and arguably simpler definition of parametric specifications;
  however, as we shall see,~\eqref{eq:pspec} highlights that a
  trace induces a structure, called the validity domain, embedded in
  the parameter space.
\end{remark}
Parametric specifications arise naturally from syntactically
substituting constants with parameters in the description of a
specification.
\begin{example}\label{ex:p_spec}
  The parametric specification given in Ex~\ref{ex:spec} can be
  generalized by substituting $\tau$ for $0.2$ and $h$ for $1$ in
  Ex~\ref{ex:spec}.
  \begin{equation}
    \label{eq:param_spec}
    \f_{ex}(\trc)(\tau, h)  \eqdef \mathbbm{1}\bigg [\forall t \in \timedomain~.~  \big (t > \tau \implies \trc(t) < h\big )\bigg](\trc)
  \end{equation}
\end{example} At this point, one could naively extend the notion of
the ``fingerprint'' of a parametric specification in a similar manner
as the finite case. However, if ${[0,1]}^n$ is equipped with a
distance metric, it is fruitful to instead study the geometry induced
by the time series in the parameter space. To begin, observe that the
value of a Boolean feature vector is exactly determined by which
entries map to $1$. Analogously, the set of parameter values for
which a parameterized specification maps to true on a given time
series acts as the ``fingerprint''. We refer to this characterizing set
as the validity domain.
\begin{definition}[Validity domain] Given an $n$ parameter specification,
$\varphi$, and a trace, $\trc$, the validity domain
is the pre-image of 1 under $\varphi(\trc)$,
\begin{equation}
    \validitydomain_\f(\trc) \eqdef \text{PreImg}_{\varphi(\trc)}[1]
    =\bigg \{\theta \in {[0, 1]}^n~|~\f(\trc)(\theta) = 1\bigg \}
\end{equation}
\end{definition}
Thus, $\validitydomain_\f$, can be viewed as the map that returns the
structure in the parameter space indexed by a particular trace.

Note that in general, the validity domain can be arbitrarily complex making
reasoning about its geometry intractable. We circumvent
such hurdles by specializing to monotonic specifications.
\begin{definition}[Monotonic Specifications]\label{def:monotonic}
  A parametric specification is said to be monotonic if for all traces,
  $\trc$:
  \begin{equation}
    \theta \pOrder \theta' \implies \f(\trc)(\theta) \leq \f(\trc)(\theta')
  \end{equation}
  where $\pOrder$ is the standard product ordering on ${[0, 1]}^n$, \eg{} $(x, y) \leq (x', y')$ iff $(x < x' \wedge y < y')$.
\end{definition}
\begin{remark}
The parametric specification in Ex~\ref{ex:p_spec} is monotonic.  
\end{remark}
\begin{proposition}
  Given a monotonic specification, $\varphi$, and a time series, $\trc$,
  the boundary of the validity domain,
  $\validitydomainboundary_\f(x)$, of a monotonic specification
  is a hyper-surface that segments ${[0, 1]}^n$ into two
  components.
\end{proposition}
In the sequel, we develop a distance metric between validity domains
which characterizes the similarity between two time series under the
lens of a monotonic specification.


\section{Logic-Respecting Distance Metric}\label{sec:metric}
In this section, we define a class of metrics on the signal space that
is derived from corresponding parametric specifications. First,
observe that the validity domains of monotonic specifications are
uniquely defined by the hyper-surface that separates them from the
rest of the parameter space. Similar to Pareto fronts in a
multi-objective optimization, these boundaries encode the trade-offs
required in each parameter to make the specification satisfied for a
given time series. This suggests a simple procedure to define a
distance metric between time series that respects their logical
properties: Given a monotonic specification, a set of time series, and
a distance metric between validity domain boundaries:
\begin{enumerate}
\item Compute the validity domain boundaries for each time series.
\item Compute the distance between the validity domain boundaries.
\end{enumerate}
Of course, the benefits of using this metric would rely entirely on
whether (i) The monotonic specification captures the relevant domain-specific details (ii) The distance between validity domain boundaries
is sensitive to outliers. While the choice of specification is highly
domain-specific, we argue that for many monotonic specifications, the
distance metric should be sensitive to outliers as this represents a
large deviation from the specification. This sensitivity requirement
seems particularly apt if the number of satisfying traces of the
specification grows linearly or super-linearly as the parameters
increase. Observing that Hausdorff distance~\eqref{eq:hausdorff}
between two validity boundaries satisfy these properties, we define
our new distance metric between time series as:
\begin{definition}
  Given a monotonic specification, $\varphi$, and a distance metric on
  the parameter space $({[0, 1]}^n, d)$, the logical distance between
  two time series, $\trc(t), \trcY(t) \in \domain^\timedomain$ is:
  \begin{equation}\label{eq:logical_distance}
    d_\varphi(\trc(t), \trcY(t)) \eqdef d_H\left (\validitydomainboundary_\varphi(\trc), \validitydomainboundary_\varphi(\trcY)\right)
  \end{equation}
\end{definition}

\subsection{Approximating the Logical Distance}
Next, we discuss how to approximate the logical distance metric 
within arbitrary precision. First, observe that the validity domain
boundary of a monotonic specification can be recursively approximated
to arbitrary precision via binary search on the diagonal of the
parameter space~\cite{maler:hal-01556243}. This approximation yields a
series of overlapping axis aligned rectangles that are guaranteed to
contain the boundary (see Fig~\ref{fig:computeBoundaries}).
\begin{figure}[H]
  \centering
  \includegraphics[height=1.3in]{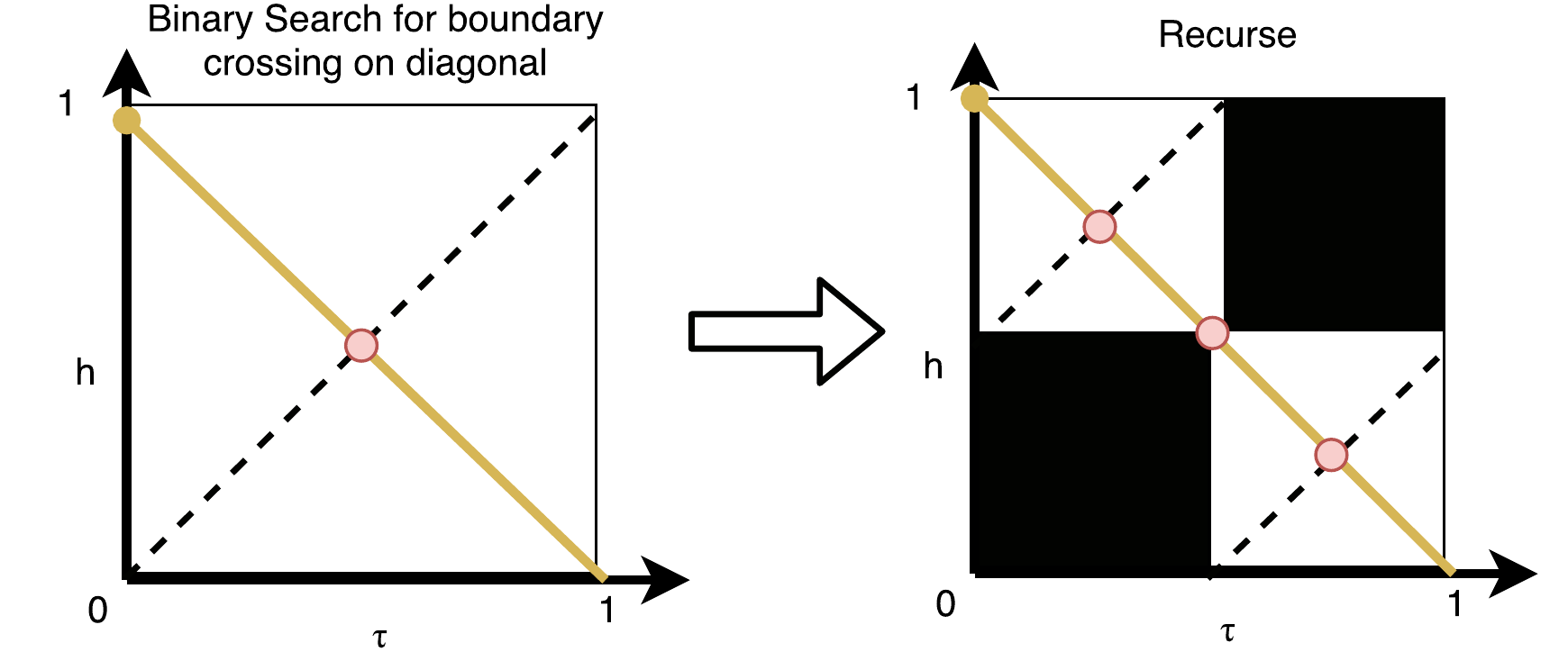}
  \caption{Illustration of procedure introduced
    in~\cite{maler:hal-01556243} to recursively approximate a validity
    domain boundary to arbitrary
    precision.}\label{fig:computeBoundaries}
\end{figure}
To formalize this approximation, let $I(\Reals)$ denote the set of
closed intervals on the real line. We then define an axis aligned
rectangle as the product of closed intervals.
\begin{definition}
  The set of axis aligned rectangles is defined as:
  \begin{equation}\label{eq:recs}
  I(\Reals^n) \eqdef \prod_{i=1}^n I(\Reals)
  \end{equation}
\end{definition}
The approximation given in~\cite{maler:hal-01556243} is then a family
of maps,
\begin{equation}\label{eq:approx_i}
  \text{approx}^i : \domain^\timedomain \to \PowerSet{I(\Reals^n)}
\end{equation}
where $i$ denotes the recursive depth $\PowerSet{\cdot}$ denotes the powerset.\footnote{
  The co-domain of~\eqref{eq:approx_i} could be tightened to $\big( 2^{n} - 2\big)^{i}$, but to avoid also parameterizing the discretization function, we do
not strengthen the type signature.}
For example, $\text{approx}^0$
yields the bounding box given in the leftmost subfigure in
Fig~\ref{fig:computeBoundaries} and $\text{approx}^1$ yields the
subdivision of the bounding box seen on the right.\footnote{ If the
rectangle being subdivided is degenerate, i.e., lies entirely within
the boundary of the validity domain and thus all point intersect the
boundary, then the halfway point of the diagonal is taken to be the
subdivision point.}

Next, we ask the question: Given a discretization of the rectangle set
approximating a boundary, how does the Hausdorff distance between the
discretization relate to the true Hausdorff distance between two
boundaries? In particular, consider the map that takes a set of
rectangles to the set of the corner points of the rectangles.
Formally, we denote this map as:
\begin{equation}\label{eq:corners}
  \text{discretize} : \PowerSet{I(\Reals^n)} \to \PowerSet{\Reals^n}
\end{equation}
As the rectangles are axis aligned, at this point, it is fruitful to
specialize to parameter spaces equipped with the infinity norm. The
resulting Hausdorff distance is denoted $d_H^\infty$. This
specialization leads to the following lemma:
\begin{lemma}\label{lem:error}
  Let $\trc$, $\trc'$ be two time series and $\mathcal{R}, \mathcal{R}'$
  the approximation of their respective boundaries. Further, let
  $p, p'$ be points in $\mathcal{R}, \mathcal{R}'$ such that:
  \begin{equation}
    \label{eq:5}
    \hat{d} \eqdef d_H^\infty(\text{discretize}(\mathcal{R}),
    \text{discretize}(\mathcal{R'})) = d_\infty(p, p')
  \end{equation}
  and let $r, r'$ be the rectangles in $\mathcal{R}$ and
  $\mathcal{R}'$ containing the points $p$ and $p'$ respectively. Finally, let
  $\frac{\epsilon}{2}$ be the maximum edge length in $\mathcal{R}$ and $\mathcal{R'}$, then:
  \begin{equation}\label{eq:4}
    \max(0, \hat{d} - \epsilon) \leq d_\f(\trc, \trc') \leq \hat{d} + \epsilon
  \end{equation}
\end{lemma}
\begin{proof}
  First, observe that (i) each rectangle intersects its boundary (ii)
  each rectangle set over-approximates its boundary.  Thus, by
  assumption, each point within a rectangle is at most $\epsilon/2$
  distance from the boundary w.r.t.\ the infinity norm. Thus, since
  there exist two points $p, p'$ such that
  $\hat{d} = d_\infty(p, p')$, the maximum deviation from the logical
  distance is at most $2\frac{\epsilon}{2} = \epsilon$ and
  $\hat{d} - 2\epsilon \leq d_\f(\trc, \trc') \leq \hat{d} +
  2\epsilon$. Further, since $d_\f$ must be in $\Reals^{\geq 0}$, the
  lower bound can be tightened to
  $\max(0, \hat{d} - 2\epsilon)$.~$\blacksquare$
\end{proof}
We denote the map given by~\eqref{eq:4} from the points to the error interval as:
\begin{equation}
  \label{eq:2}
 d_H^\infty \pm \epsilon : \PowerSet{\Reals} \times \PowerSet{\Reals} \to I(\mathbb{R}^+)
\end{equation}
Next, observe that this approximation can be made arbitrarily close to
the logical distance.
\begin{theorem}
  Let $d^\star= d_\f(\trc, \trcY)$ denote the logical distance between two
  traces $\trc, \trcY$.
  For any $\epsilon \in \Reals^{\geq 0}$, there exists $i\in \Nat$ such that:
  \begin{equation}
    \label{eq:arb_precision}
    d_H^\infty(\text{discretize}(\text{approx}^i(\mathcal{R})),
    \text{discretize}(\text{approx}^i(\mathcal{R}'))) \in [d^\star - \epsilon, d^\star + \epsilon]
  \end{equation}
\end{theorem}
\begin{proof}
  By Lemma~\ref{lem:error}, given a fixed approximate depth, the above
  approximation differs from the true logical distance by at most two
  times the maximum edge length of the approximating rectangles.
  Note that by construction, incrementing the approximation depth
  results in each rectangle having at least one edge being
  halved. Thus the maximum edge length across the set of rectangles
  must at least halve. Thus, for any $\epsilon$ there exists an approximation
  depth $i\in \Nat$ such that:
  \[
    d_H^\infty(\text{discretize}(\text{approx}^i(\mathcal{R})),
    \text{discretize}(\text{approx}^i(\mathcal{R}'))) \in [d^\star - \epsilon, d^\star + \epsilon]~.~
   \blacksquare
  \]

\end{proof}
Finally, algorithm~\ref{alg:approx} summarizes the above procedure.
\begin{algorithm}[H]
  \scriptsize
  \caption{Approximate Logical Distance\label{alg:approx}}
  \begin{algorithmic}[1]
    \Procedure{approx\_dist}{$\trc, \trc', \delta$}
    \State{$lo, hi \gets 0, \infty$}
    \While{$hi - lo > \delta$}
    \State{$\mathcal{R}, \mathcal{R}' \gets \text{approx}^i(\trc), \text{approx}^i(\trc')$}
    \State{$\text{points}, \text{points}' \gets \text{discretize}(\mathcal{R}), \text{discretize}(\mathcal{R}')$}
    \State{$lo, hi \gets \big (d_H^\infty \pm \epsilon\big )(\mathcal{R}, \mathcal{R}')$}
    \EndWhile{}
    \State{\Return{$lo, hi$}}
    \EndProcedure{}
  \end{algorithmic}
\end{algorithm}

\begin{remark}
  An efficient implementation should of course memoize previous calls
to $\text{approx}^i$ and use $\text{approx}^i$ to compute
$\text{approx}^{i+1}$. Further, since certain rectangles can be
quickly determined to not contribute to the Hausdorff distance, they
need not be subdivided further.
\end{remark}

\subsection{Learning Labels}
The distance interval $(lo, hi)$ returned by Alg~\ref{alg:approx} can
be used by learning techniques, such as
\textit{hierarchical or agglomerative clustering}, to estimate
clusters (and hence the labels). While the technical details of these
learning algorithms are beyond the scope of this work, we formalize
the result of the learning algorithms as a labeling map:
\begin{definition}[Labeling]
  A $k$-labeling is a map:
  \begin{equation}
    \label{eq:labeling}
    L : \domain^\timedomain \to \{0, \ldots, k\}
  \end{equation}
  for some $k \in \Nat$. If $k$ is obvious from context or not
  important, then the map is simply referred to as a
  labeling.
\end{definition}
 


\section{Artifiact Extraction}
In practice, many learning algorithms produce labeling maps that
provide little to no insight into why a particular trajectory is given
a particular label. In the next section, we seek a way to
systematically summarize a labeling in terms of the parametric
specification used to induce the logical distance.
\subsection{Post-facto Projections}
To begin, observe that due to the nature of the Hausdorff distance,
when explaining why two boundaries differ, one can remove large
segments of the boundaries without changing their Hausdorff distance. This
motivates us to find a small summarizing set of parameters for
each label. Further, since the Hausdorff distance often reduces to
the distance between two points, we aim to summarize each boundary
using a particular projection map. Concretely,
\begin{definition}
  Letting $\validitydomainboundary_\f(\domain^\timedomain)$ denote the
  set of all possible validity domain boundaries, a projection is a
  map:
  \begin{equation}
    \label{eq:proj}
    \pi : \validitydomainboundary_\f(\domain^\timedomain) \to \Reals^n
  \end{equation}
  where $n$ is the number of parameters in $\f$.
\end{definition}
\begin{remark}
  In principle, one could extend this to projecting to a finite tuple
  of points. For simplicity, we do not consider such cases.
\end{remark}
Systematic techniques for picking the projection include
\textit{lexicographic projections} and solutions to
\textit{multi-objective optimizations}; however, as seen in the
introduction, a-priori choosing the projection scheme is subtle.
Instead, we propose performing a post-facto optimization of a
collection of projections in order to be maximally representative of
the labels. That is, we seek a projection, $\pi^*$, that maximally
disambiguates between the labels, i.e., maximizes the minimum distance
between the clusters. Formally, given a set of traces associated with
each label $L_1, \ldots L_k$ we seek:
\begin{equation}
  \label{eqn:proj}
  \pi^* \in \argmax_{\pi} \min_{i,j \in {k\choose 2}}  d_{\infty}(\pi(L_i), \pi(L_j))
\end{equation}
For simplicity, we restrict our focus to projections induced by the
intersection of each boundary with a line intersecting the base of the
unit box ${[0, 1]}^n$.  Just as in the recursive boundary
approximations, due to monotonicity, this intersection point is
guaranteed to be unique. Further, this class of projections is in
one-one correspondence with the boundary. In particular, for any point
$p$ on boundary, there exists exactly one projection that produces
$p$. As such, each projection can be indexed by a point in ${[0,
1]}^{n-1}$.

This is perhaps unsurprising given that in 2-d, one can index this
class by the angle with the $x$-axis and in 3-d on can include the
angle from the $z$-axis.
\begin{remark}
  Since we expect clusters of boundaries to be near each other, we
  also expect their intersection points to be near each other.
\end{remark}
\begin{remark}
  For our experiment, we searched for $\pi^*$ via a sweep through a
  discretized indexing of possible angles.
\end{remark}
\subsection{Label Specifications}
Next, observe that given a projection, when studying the infinity
norm distance between labels, it suffices to consider only the
bounding box of each label in parameter space. Namely,
letting $B : \PowerSet{\Reals^n} \to I[\Reals^n]$ denote
the map that computes the bounding box of a set of points in $\Reals^n$,
for any two labels $i$ and $j$:
\begin{equation}
  \label{eq:interpBounding}
  d_{\infty}(\pi(L_i), \pi(L_j)) =d_{\infty}(B\circ\pi(L_i), B\circ \pi(L_j)).
\end{equation}
This motivates using the projection's bounding box as a surrogate for
the cluster. Next, we observe that one can encode the set of
trajectories whose boundaries intersect (and thus can project to) a
given bounding box as a simple Boolean combination of the
specifications corresponding to instantiating $\f$ with the parameters
of at most $n+1$ corners of the box~\cite[Lemma
2]{logicalClustering}. While a detailed exposition is outside the
scope of this article, we illustrate with an example.
\begin{example}
  Consider examples 0 and 1 from the introductory example viewed
  as validity domain boundaries under~\eqref{eq:param_spec}. Suppose
  that the post-facto projection mapped example 0 to $(1/4, 1/2)$ and
  mapped example 1 to $(0.3, 0.51)$. Such a projection is plausibly near
  the optimal for many classes of projections since none of the other
  example boundaries (who are in different clusters) are near the
  boundaries for 0 and 1 at these points. The resulting specification
  is:
  \begin{equation}
    \begin{split}
      \phi&(\trc) = \f_{ex}(\trc)(1/4, 1/2)\wedge \neg\f_{ex}(\trc)(1/4, 0.51)\wedge \neg\f_{ex}(\trc)(0.3, 1/2)\\
      &= \mathbbm{1}\bigg[t \in [1/4, 0.3] \implies  \trc(t) \in [1/2, 0.51] \wedge t > 0.3 \implies \trc(t) \geq 0.51 \bigg]
    \end{split}
  \end{equation}
\end{example}

\subsection{Dimensionality Reduction}
  \begin{wrapfigure}{r}{0.5\textwidth}
    \centering
    \includegraphics[width=0.5\textwidth]{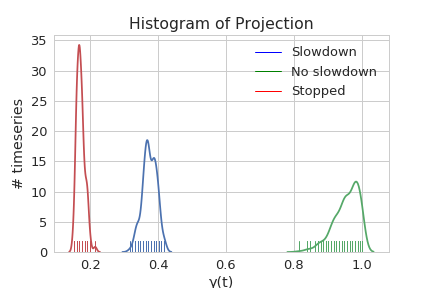}
    \caption{Figure of histogram resulting from projecting
      noisy variations of the traffic slow down example time
      series onto the diagonal of the unit box.}\label{fig:dim_reduce}
  \end{wrapfigure}
Finally, observe that the line that induces the projection can serve
as a mechanism for dimensionality reduction. Namely, if one
parameterizes the line $\gamma(t)$ from $[0, 1]$, where $\gamma(0)$ is
the origin and $\gamma(1)$ intersects the unit box, then the points
where the various boundaries intersect can be assigned a number
between $0$ and $1$. For high-dimensional parameter spaces, this
enables visualizing the projection histogram and could even be used
for future classification/learning. We again illustrate using our
running example. 
\begin{example}
  For all six time series in the traffic slow down example, we
  generate 100 new time series by modulating the time series with
  noise drawn from $\mathcal{N}(1, 0.3)$. Using our previously labeled
  time series, the projection using the line with angle $45^\circ$ (i.e.,
  the diagonal of the unit box) from the x-axis yields the
  distribution seen in Fig~\ref{fig:dim_reduce}. Observe that all three
  clusters are clearly visible.
\end{example}

\begin{remark}\label{rem:pca}
  If one dimension is insufficient, this procedure can be extended to
  an arbitrary number of dimensions using more lines. An interesting
  extension may be to consider how generic dimensionality techniques
  such as principle component analysis would act in the limit where
  one approximates the entire boundary.
\end{remark}

\section{Case Study}\label{sec:casestudies}
To improve driver models and traffic on highways, the Federal Highway
Administration collected detailed traffic data on southbound US-101
freeway, in Los Angeles~\cite{NGSIM}. Traffic through the segment was
monitored and recorded through eight synchronized cameras, next to the
freeway. A total of $45$ minutes of traffic data was recorded
including vehicle trajectory data providing lane positions of each
vehicle within the study area. The data-set is split into $5979$ time
series. For simplicity, we constrain our focus to the car's speed. In
the sequel, we outline a technique for first using the parametric
specification (in conjunction with off-the-shelf machine learning
techniques) to filter the data, and then using the logical distance
from an idealized slow down to find the slow downs in the data. This
final step offers a key benefit over the closest prior
work~\cite{logicalClustering}. Namely given an over approximation of
the desired cluster, one can use the logical distance to further
refine the cluster.

\mypara{Rescale Data} As in our running example, we seek to
use~\eqref{eq:param_spec} to search for traffic slow downs; however,
in order to do so, we must re-scale the time series. To begin, observe
that the mean velocity is 62mph with 80\% of the vehicles remaining
under 70mph. Thus, we linearly scale the velocity so that
$70\text{mph} \mapsto 1\text{ arbitrary unit (a.u.)}$. Similarly, we
re-scale the time axis so that each tick is $2$
seconds. Fig~\ref{fig:normalized} shows a subset of the time series.
\begin{figure}[h]
  \centering
  \begin{subfigure}[t]{0.49\textwidth}
    \centering
    \includegraphics[width=\linewidth]{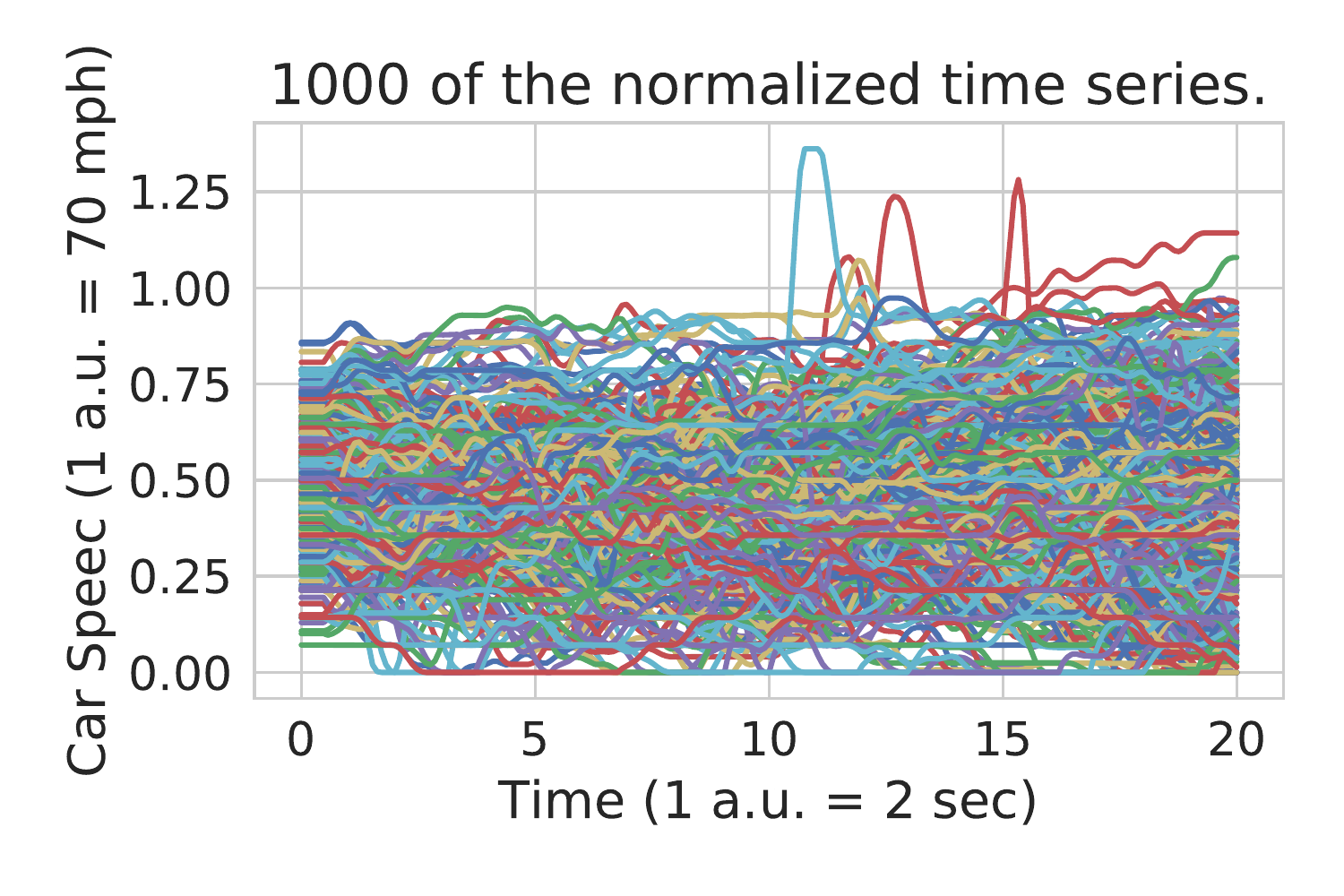}
    \caption{1000 / 5000 of the rescaled highway 101 time
      series.}\label{fig:normalized}
  \end{subfigure}
  \begin{subfigure}[t]{0.49\textwidth}
    \centering \includegraphics[width=\textwidth]{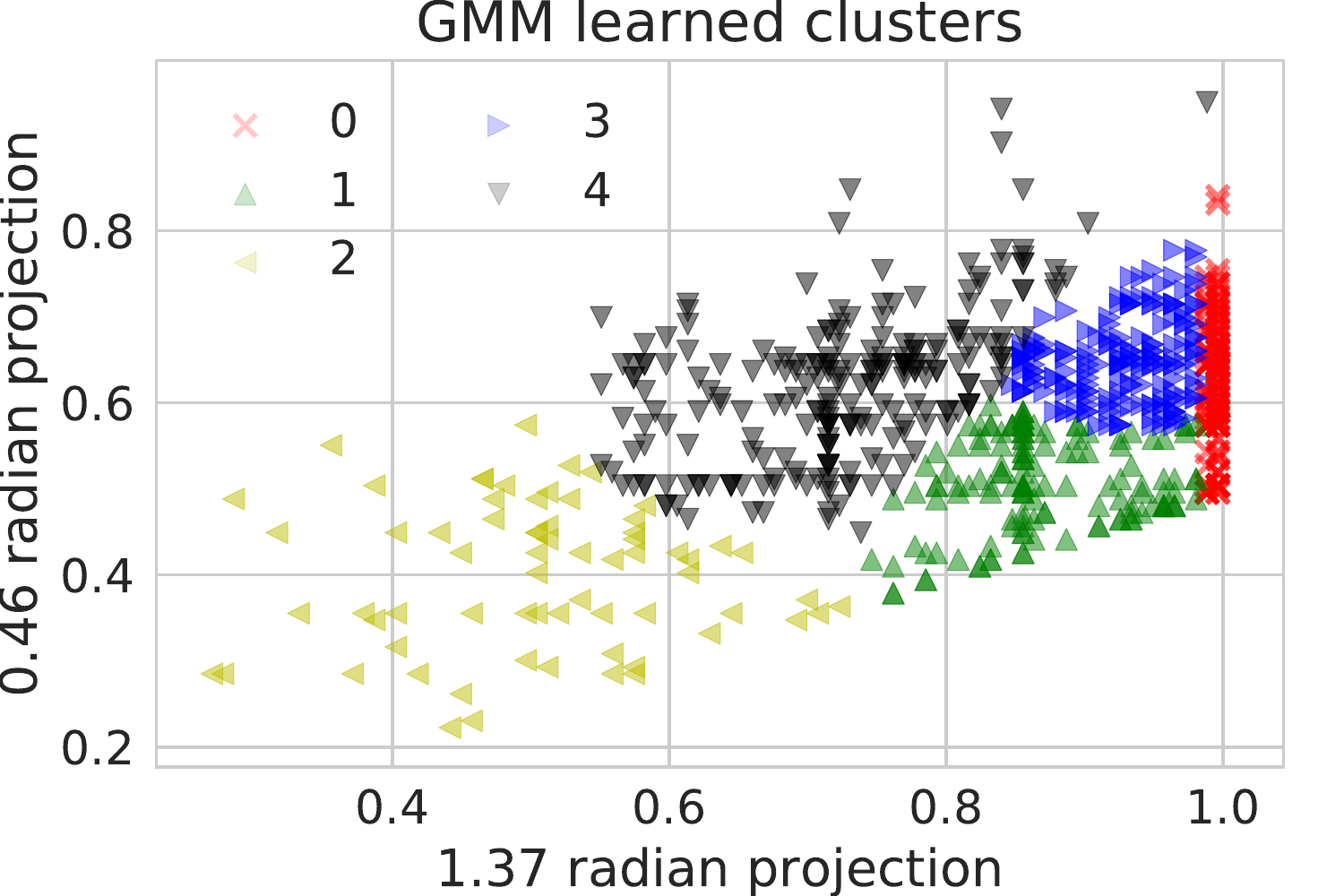}
    \caption{Projection of Time-Series to two lines in the parameter
      space of~\eqref{eq:param_spec} and resulting GMM
      labels. }\label{fig:gmm}
  \end{subfigure}
\end{figure}
  
\mypara{Filtering} Recall that if two boundaries have small Hausdorff
distances, then the points where the boundaries intersect a line (that
intersects the origin of the parameter space) must be close. Since
computing the Hausdorff distance is a fairly expensive operation, we
use this one way implication to group time series which may be near
each other w.r.t.\ the Hausdorff distance.

In particular, we (arbitrarily) selected two lines intersecting the
parameter space origin at $0.46$ and $1.36$ radians from the $\tau$
axis to project to. We filtered out any time-series that did not
intersect the line within ${[0, 1]}^2$. We then fit a 5 cluster
Gaussian Mixture Model (GMM) to label the data. Fig~\ref{fig:gmm}
shows the result.

\begin{figure}[h]
  \centering
  \begin{subfigure}[t]{0.49\textwidth}
    \centering
    \includegraphics[width=\linewidth]{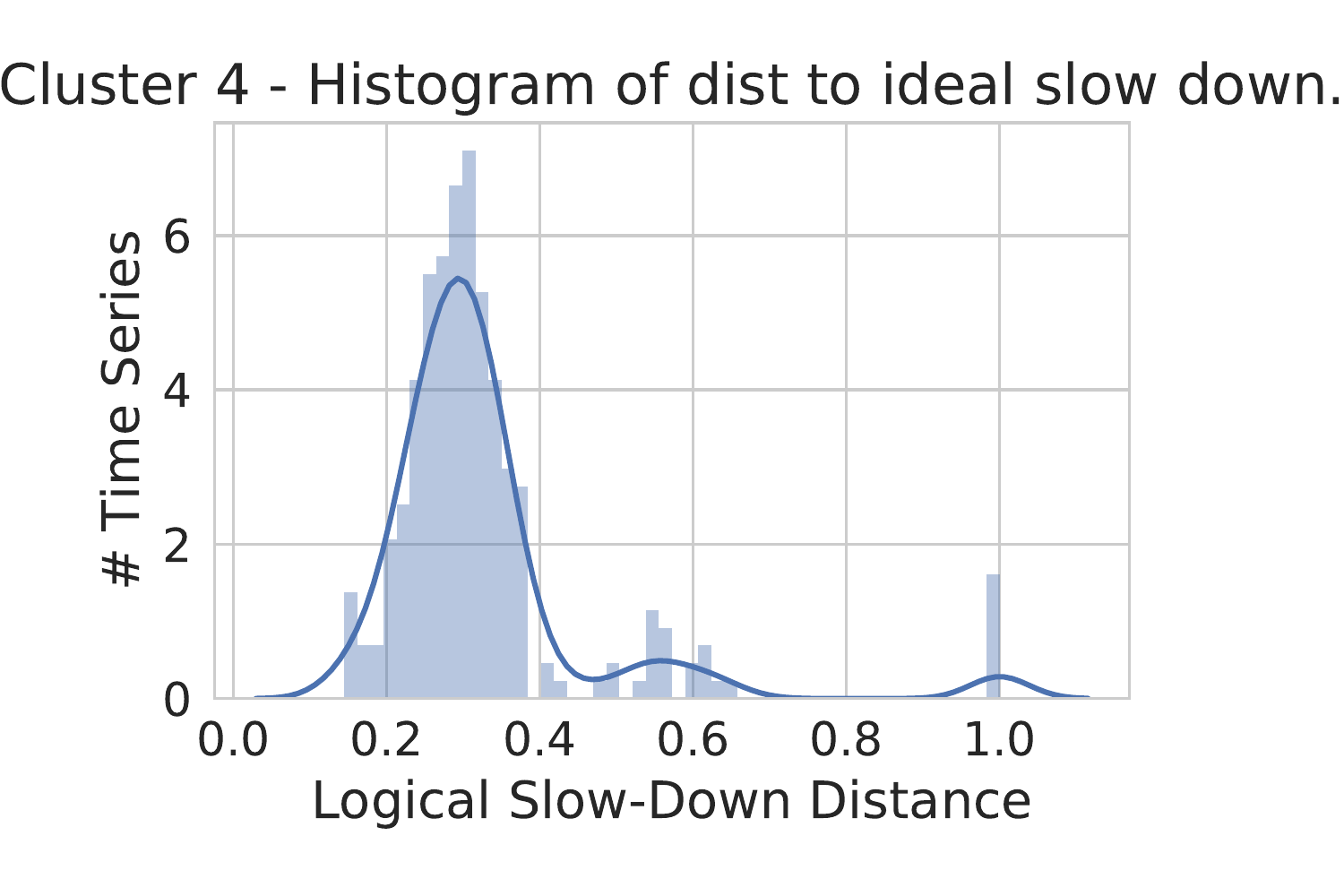}
    \caption{Cluster~4 Logical distance
      histogram.}\label{fig:c4_dist_dist}
  \end{subfigure}
  \begin{subfigure}[t]{0.49\textwidth}
    \centering
    \includegraphics[width=\textwidth]{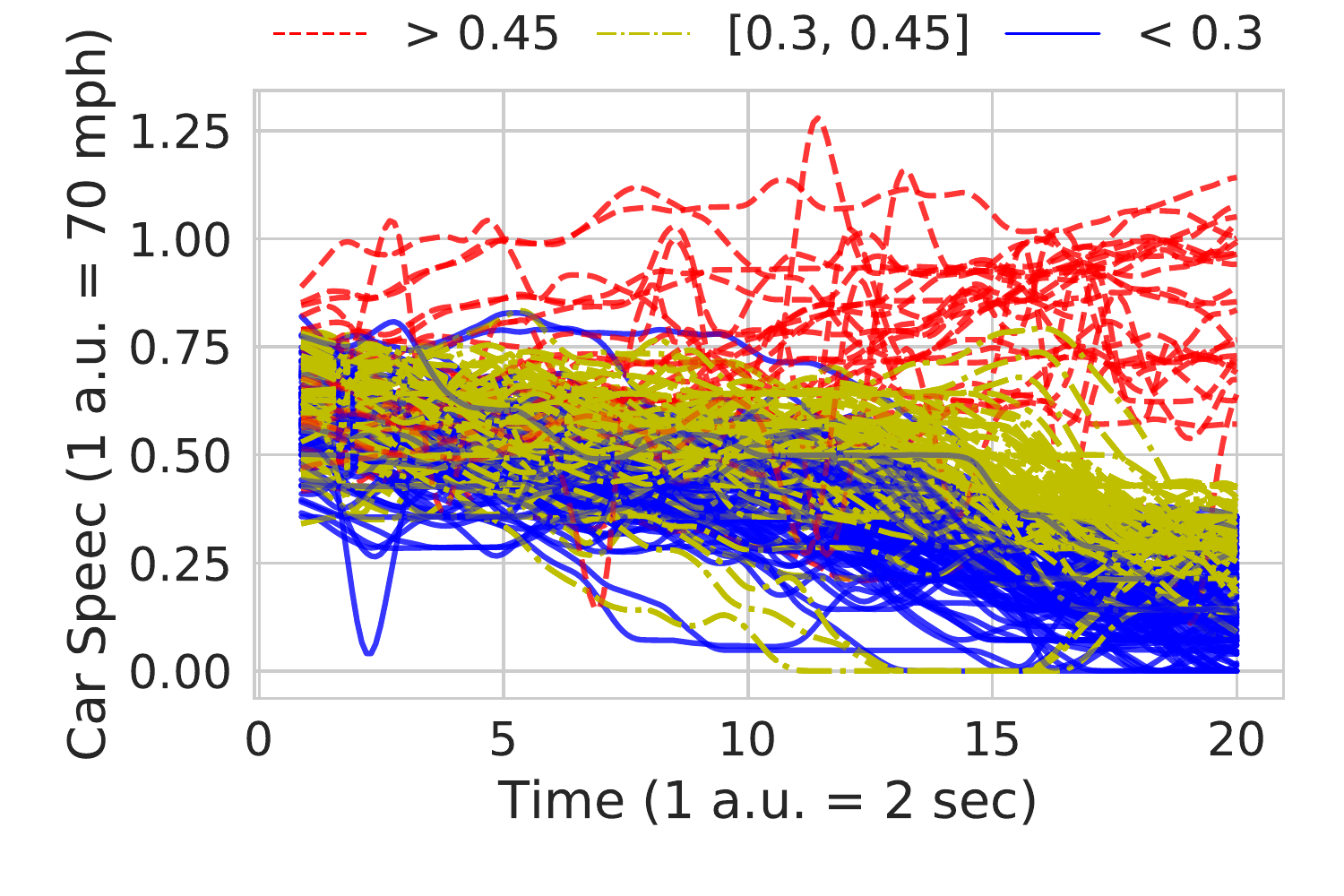}
    \caption{Time-series in Cluster 4 colored by distance to ideal
      slow down.}\label{fig:annotated_c4}
  \end{subfigure}
\end{figure}

\mypara{Matching Idealized Slow Down} Next, we labeled the idealized
slow down, (trace 0 from Fig~\ref{fig:toy_boundaries}) using the
fitted GMM.\@ This identified cluster~4 (with 765 data points) as
containing potential slow downs. To filter for the true slow downs, we
used the logical distance\footnote{again associated
  with~\eqref{eq:param_spec}} from the idealized slow down to further
subdivide the cluster. Fig~\ref{fig:c4_dist_dist} shows the resulting
distribution. Fig~\ref{fig:annotated_c4} shows the time series in
cluster 4 annotated by their distance for the idealized slow down.
Using this visualization, one can clearly identify 390 slow downs
(distance less than $0.3$)

\mypara{Artifact Extraction} Finally, we first searched for a single
projection that gave a satisfactory separation of clusters, but were
unable to do so. We then searched over pairs of projections to create
a specification as the conjunction of two box specifications.  Namely,
in terms of $\eqref{eq:param_spec}$, our first projection yields the
specification: $\phi_1 = \f_{ex}(0.27, 0.55) \wedge \neg \f_{ex}(0.38,
0.55) \wedge \neg \f_{ex}(0.27, 0.76)$. Similarly, our second
projection yields the specification: $\phi_2 = \f_{ex}(0.35, 0.17)
\wedge \neg \f_{ex}(0.35, 0.31) \wedge \neg \f_{ex}(0.62, 0.17)$. The
learned slow down specification is the conjunction of these two
specifications.




\section{Related Work}

Time-series clustering and classification is a well-studied area in
the domain of machine learning and data
mining~\cite{liao2005clustering}. Time series clustering that work
with raw time-series data combine clustering schemes such as
agglomerative clustering, hierarchical clustering, $k$-means
clustering among others, with similarity measures between time-series
data such as the dynamic time-warping (DTW) distance, statistical
measures and information-theoretic measures.  Feature-extraction based
methods typically use generic sets of features, but algorithmic
selection of the right set of meaningful features is a
challenge. Finally, there are model-based approaches that seek an
underlying generative model for the time-series data, and typically
require extra assumptions on the data such as linearity or the
Markovian property.  Please see~\cite{liao2005clustering} for detailed
references to each approach. It should be noted that historically
time-series learning focused on univariate time-series, and extensions
to multivariate time-series data have been relatively recent
developments.

More recent work has focused on automatically identifying features
from the data itself, such as the work on {\em
  shapelets\/}~\cite{ye2009time,mueen2011logical,lines2012shapelet},
where instead of comparing entire time-series data using similarity
measures, algorithms to automatically identify distinguishing motifs
in the data have been developed. These motifs or shapelets serve not
only as features for ML tasks, but also provide visual feedback to the
user explaining why a classification or clustering task, labels given
data, in a certain way. While we draw inspiration from this general
idea, we seek to expand it to consider logical shapes in the data,
which would allow leveraging user's domain expertise.

Automatic identification of motifs or basis functions from the data
while useful in several documented case studies, comes with some
limitations. For example, in~\cite{bahadori2015functional}, the
authors define a subspace clustering algorithm, where given a set of
time-series curves, the algorithm identifies a subspace among the
curves such that every curve in the given set can be expressed as a
linear combination of a deformations of the curves in the subspace. We
note that the authors observe that it may be difficult to associate
the natural clustering structure with specific predicates over the
data (such as patient outcome in a hospital setting).

The use of logical formulas for learning properties of time-series has
slowly been gaining momentum in communities outside of traditional
machine learning and data
mining~\cite{bartocci2014data,dTree,rPSTL,iPSTL}. Here, fragments of
Signal Temporal Logic have been used to perform tasks such as
supervised and unsupervised learning. A key distinction from these
approaches is our use of libraries of signal predicates that encode
domain expertise that allow human-interpretable clusters and
classifiers.

Finally, preliminary exploration of this idea appeared in prior work
by some of the co-authors in~\cite{logicalClustering}. The key
difference is the previous work required users to provide a ranking of
parameters appearing in a signal predicate, in order to project
time-series data to unique points in the parameter space. We remove
this additional burden on the user in this paper by proposing a
generalization that projects time-series signals to trade-off curves
in the parameter space, and then using these curves as features.


\mypara{Conclusion}
We proposed a family of distance metrics for time-series learning
centered {\em monotonic specifications\/} that respect the logical
characteristic of the specification. The key insight was to first map
each time-series to characterizing surfaces in the parameter space and
then compute the Hausdorff Distance between the surfaces. This enabled
embedding non-trivial domain specific knowledge into the distance
metric usable by standard machine learning. After labeling
the data, we demonstrate how this technique produces artifacts that
can be used for dimensionality reduction or as a logical specification
for each label. We concluded with a simple automotive case study
show casing the technique on real world data. Future work includes
investigating how to the leverage massively parallel natural in the boundary and Hausdorff computation using graphical processing units and characterizing
alternative boundary distances (see Remark~\ref{rem:pca}).


%
%
%
%
\bibliographystyle{splncs04}
\bibliography{refs}

\end{document}